\def\year{2020}\relax
\documentclass[letterpaper]{article} 
\usepackage{aaai20}  
\usepackage{times}  
\usepackage{helvet} 
\usepackage{courier}  
\usepackage{url}  
\usepackage{graphicx} 
\usepackage{amsmath}
\usepackage{amssymb}
\usepackage[linesnumbered,ruled,vlined]{algorithm2e}
\usepackage{ifthen}
\usepackage{mathrsfs}
\usepackage{dsfont}
\usepackage{amsthm}
\usepackage{graphicx}

\DeclareMathAlphabet{\mathbfit}{OML}{cmm}{b}{it}

\long\def\dl{|\hspace{-.03cm}|}
\long\def\db#1{[\![#1]\!]}

\newtheorem{thm}{Theorem}
\newtheorem{cor}[thm]{Corollary}
\newtheorem{lem}[thm]{Lemma}
\newtheorem{prop}[thm]{Proposition}

\theoremstyle{definition}
\newtheorem{defn}{Definition}
\theoremstyle{remark}
\newtheorem{rem}{Remark}
\theoremstyle{definition}
\newtheorem{exm}{Example}

\title{Towards Universal Languages for Tractable Ontology Mediated Query Answering}
\author{Heng Zhang,\textsuperscript{\rm 1} Yan Zhang,\textsuperscript{\rm 2,3} Jia-Huai You,\textsuperscript{\rm 4} Zhiyong Feng,\textsuperscript{\rm 1} Guifei Jiang\,\textsuperscript{\rm 5}\\
\\
\textsuperscript{\rm 1}College of Intelligence and Computing, Tianjin University, Tianjin, China\\
\textsuperscript{\rm 2}School of Computing, Engineering and Mathematics, Western Sydney University, Penrith, Australia\\
\textsuperscript{\rm 3}School of Computer Science and Technology, Huazhong University of Technology and Science, Wuhan, China\\
\textsuperscript{\rm 4}Department of Computing Science, University of Alberta, Edmonton, Canada\\
\textsuperscript{\rm 5}College of Software, Nankai University, Tianjin, China\\
heng.zhang@tju.edu.cn, yan.zhang@westernsydney.edu.au, jyou@ualberta.ca, zyfeng@tju.edu.cn, g.jiang@nankai.edu.cn
}
\begin{document}

\maketitle

\begin{abstract}
An ontology language for ontology mediated query answering (OMQA-language) is universal for a family of OMQA-languages if it is the most expressive one among this family. In this paper, we focus on three families of tractable OMQA-languages, including first-order rewritable languages and languages whose data complexity of the query answering is in {AC}$^0$ or PTIME. On the negative side, we prove that there is, in general, no universal language for each of these families of languages. On the positive side, we propose a novel property, the locality, to approximate the first-order rewritability, and show that there exists a language of disjunctive embedded dependencies that is universal for the family of OMQA-languages with locality. All of these results apply to OMQA with query languages such as conjunctive queries, unions of conjunctive queries and acyclic conjunctive queries.
\end{abstract}

\section{Introduction}

Ontology mediated query answering (OMQA) is a paradigm that generalizes the traditional database querying by enriching the database with a domain ontology~\cite{PoggiLCGLR08}. This paradigm has played an important role in the semantic web~\cite{calvanese:DL-lite2007,lutz:ijcai05}, data modelling~\cite{BerardiCG05}, data exchange~\cite{FKMP05} and data integration~\cite{Lenzerini02}, and has recently emerged as one of the central issues in knowledge representation as well as in databases. 

A long-term major topic for OMQA is to identify proper languages that specify ontologies. There have been a large number of ontology languages proposed for OMQA since the mid 2000s. For instance, in description logics, the DL-Lite family~\cite{calvanese:DL-lite2007}, $\mathcal{EL}$-family~\cite{lutz:ijcai05} and other variants have been proposed and extensively studied. More recently, the Datalog${\pm}$ family, a.k.a. existential rule languages, or dependencies in databases, have been rediscovered as promising languages for OMQA, see, e.g.,~\cite{BLMS11,CaliGL12,CaliGP12}. Most of these languages enjoy good computational properties such as the first-order rewritability or {PTIME} data complexity.

While all these languages are of their specific features and hence are useful in different applications, it is not realistic to implement OMQA-systems for all of them. So a natural question arises: Can we find the largest one (in the expressiveness) among the family of first-order rewritable (or PTIME-tractable) OMQA-languages? Let us call the largest language in the above sense a {\em universal language}. Clearly, it is of great theoretical and practical importance to identify the existence of universal language w.r.t. some kind of tractability, which is also the main task of this paper.

It is worth noting that the universality is one of the major principles for designing languages in both computer science and logic. For example, almost all the traditional programming languages, including C, Java and Prolog, are known to be universal for the family of Turing complete programming languages; propositional logic can express all boolean functions; and by the well-known Lindstr\"{o}m theorem first-order logic is the largest one among logics that enjoy both the compactness and the L\"{o}wenheim-Skolem property; see, e.g.,~\cite{EbbinghuasFT1994}. In databases, first-order language is shown to be universal for the family of query languages with data complexity in {AC}$^0$, and Datalog universal for the family of query languages with data complexity in PTIME; see, e.g.,~\cite{Immerman99}.

Some recent work in OMQA has been done along the line of identifying universal languages. \citeauthor{CalvaneseGLLR13}~\shortcite{CalvaneseGLLR13} proved that, under a certain syntactic classification, some languages in the DL-Lite family are the maximal fragments of description logic with the first-order rewritability. By regarding OMQA as traditional database querying, \citeauthor{GRS2014}~\shortcite{GRS2014} showed that weakly-guarded tuple-generating dependencies (TGDs) capture the class of EXPTIME queries; \citeauthor{RudolphT2015}~\shortcite{RudolphT2015} proved that general TGDs capture the class of recursively enumerable queries. In the setting of schema mapping, \citeauthor{ZhangZY15}~\shortcite{ZhangZY15} showed that the language of weakly-acyclic TGDs is universal for languages of TGDs with finite semi-oblivious chase. All of these results shed new insights on understanding the expressiveness of existential rules, but it is worth noting that OMQA is significantly different from both traditional database querying and schema mapping. To understand the expressiveness in the framework of OMQA, \citeauthor{ZhangZY16}~\shortcite{ZhangZY16} proved that the language of disjunctive embedded dependencies is universal for the family of recursively enumerable OMQA-languages. Along this line, this paper will focus on tractable OMQA.

Aimed at exploiting universal languages for the tractable OMQA, in this paper we focus on three families of OMQA-languages, including first-order rewritable languages and languages whose data complexity is in AC$^0$ or PTIME. Our contributions are summarized as follows. On one hand, we prove that there is, in general, no universal language for each of the above families of languages. On the other hand, by restricting the number of database constants involved in query answering, we propose a novel property, called the {\em locality}, to approximate the first-order rewritability, and identify the existence of universal language for the family of local OMQA-languages. All of these results hold for OMQA with query languages such as conjunctive queries, unions of conjunctive queries and acyclic conjunctive queries.

\section{Preliminaries}

{\noindent\bf Databases and Instances.} We use a countably infinite set 
of {\em constants} and a countably infinite set 
of {\em variables}, and assume they are disjoint. Every {\em term} is either a constant or a variable. A {\em relational schema} $\mathscr{R}$ consists of a set of {\em relation symbols}. Each relation symbol has an {\em arity} which is a natural number. An {\em atoms over $\mathscr{R}$} (or $\mathscr{R}$-atom) is either an equality, or a {\em relational atom} built upon terms and relation symbols in $\mathscr{R}$. A {\em fact} is a variable-free relational atom. Each {\em instance over $\mathscr{R}$} (or {\em $\mathscr{R}$-instance}) consists of a set of facts over $\mathscr{R}$. Instances that are finite are called {\em databases}. Suppose $I$ is an instance. Let $adom(I)$ denote the set of constants that occur in $I$. Let {DB}$[\mathscr{R}]$ denote the class of all databases over schema $\mathscr{R}$. Given a set $A$ of constants, by $I|_A$ we denote the subset of $I$ in which each fact involves only constants in $A$.

Let $I$ and $J$ be instances over a relational schema $\mathscr{R}$, and $C\subseteq adom(I)\cap adom(J)$. Then every {\em $C$-homomorphism} from $I$ to $J$ is a function $h:adom(I)\rightarrow adom(J)$ such that (i)~$R(\vec{a})\in I$ implies $R(h(\vec{a}))\in J$ for all relation symbols $R\in\mathscr{R}$ and all tuples $\vec{a}$ of constants, and (ii)~$h(c)=c$ for all $c\in C$. If such $h$ exists, we say that $I$ is {\em $C$-homomorphic} to $J$, and write $I\rightarrow_C J$; in addition, we write $I\twoheadrightarrow_C J$ if $h$ is injective. For simplicity, $C$ will be dropped if it is empty. 

\medskip
{\noindent\bf Queries.} Fix $\mathscr{R}$ as a relational schema. By a {\em query} over $\mathscr{R}$ (or {\em$\mathscr{R}$-query}) we mean a formula built upon atoms over $\mathscr{R}$ in some logic. The logic could be first-order logic, second-order logic, or other variants. A query is {\em boolean} if it has no free variables. For convenience, given any query $q$, let $const(q)$ denote the set of constants that occur in $q$.

Every first-order formula is called a {\em first-order query}. A {\em conjunctive query (CQ)} is a query of the form $\exists\vec{y}.\varphi(\vec{x},\vec{y})$ where $\varphi$ is a finite but nonempty conjunction of relational atoms. Let $q$ be a boolean CQ. We use $[q]$ to denote the database that consists of all the atoms that appear in $q$, where variables in atoms are regarded as special constants. The {\em Gaifman graph} of ${q}$ is an undirected graph with each term in ${q}$ as a vertex, and with each pair of distinct terms as an edge if they cooccur in some atom in ${q}$. A boolean CQ is called {\em acyclic} if its Gaifman graph is acyclic. A {\em union of conjunctive query (UCQ)} is a first-order formula built upon atoms by connectives $\wedge,\vee$ and quantifier $\exists$. Clearly, every UCQ is equivalent to a disjunction of CQs.

Every {\em Datalog$^\neg$ program} consists of a finite set of {\em rules} of the form $\forall\vec{x}\forall\vec{y}(\varphi(\vec{x},\vec{y})\rightarrow\alpha(\vec{x}))$, where $\alpha$ is a relational atom and $\varphi$ is a finite conjunction of atoms or negated atoms; $\alpha$ and $\varphi$ are called the {\em head} and the {\em body} of the rule, respectively. Each variable in $\vec{x}$ should have at least one positive occurrence in $\varphi$. A relation symbol is called {\em intentional} if it has at least one occurrence in the head of some rule, and {\em extensional} otherwise. No intensional relation symbol is allowed to appear in a negated atom. A {\em Datalog$^\neg$ query} is of the form $(\Pi,\textit{P})(\vec{x})$ where $\Pi$ is a Datalog$^\neg$ program, $\textit{P}$ an extensional relation symbol, and $\vec{x}$ a variable tuple of a proper length. It is well-known that every Datalog$^\neg$ query can be translated to an equivalent formula in least fixpoint logic, see, e.g.,~\cite{EbbinghuasF95}.

Only boolean queries will be used in this work.
For convenience, we employ {CQ}, {ACQ} and {UCQ} to denote the classes of boolean CQs, boolean acyclic CQs and boolean UCQs, respectively. Let {FO} denote the class of boolean first-order queries,  {FO}$(+,\times)$ denote the class of boolean first-order queries that involve two built-in arithmetic relations $+$ and $\times$, and {Datalog}$^\neg(\le)$ denote the class of boolean Datalog$^\neg$ queries that involve a built-in successor relation $\textit{Succ}$, and special constants $\textit{min}$ and $\textit{max}$, denoting the minimum and the maximum elements, respectively, under the underlying order. Given a class $\mathcal{C}$ of queries and a relational schema $\mathscr{R}$, let $\mathcal{C}[\mathscr{R}]$ denote the class of $\mathscr{R}$-queries that belong to $\mathcal{C}$. 

In the theory of descriptive complexity~\cite{Immerman99}, it was proved that {FO}$(+,\times)$ and {Datalog}$^\neg(\le)$ exactly capture complexity classes {AC}$^0$ and PTIME respectively, where {AC}$^0$ denotes the class of languages recognized by a uniform family of circuits with constant depth and polynomial size, and {PTIME} denotes the class of languages recognized by a deterministic Turing machine in polynomial time.

\medskip
{\noindent\bf Dependencies.} A {\em disjunctive embedded dependency (DED)} over a relational schema $\mathscr{R}$ is a sentence $\sigma$ of the form
\begin{equation*}
\forall\vec{x}\forall\vec{y}(\phi(\vec{x},\vec{y})\rightarrow\exists\vec{z}_1.\psi_1(\vec{x},\vec{z}_1)\vee\cdots\vee\exists\vec{z}_k.\psi_k(\vec{x},\vec{z}_k))
\end{equation*}
where $k\ge 0$, $\phi$ is a conjunction of relational $\mathscr{R}$-atoms involving terms from $\vec{x}\cup\vec{y}$ only, each $\psi_i$ is a conjunction of atoms over $\mathscr{R}$ involving terms from $\vec{x}\cup\vec{z}_i$ only, and each variable in $\vec{x}$ has at least one occurrence in $\phi$. In particular, $\sigma$ is called a {\em tuple-generating dependency (TGD)} if it is equality-free and $k=1$. For simplicity, we omit the universal quantifiers and the brackets appearing outside the atoms.

Let $D$ be a database, $\Sigma$ a set of (first-order) sentences, and ${q}$ a boolean query; all of them are over a common relational schema $\mathscr{R}$. We write $D\cup\Sigma\vDash{q}$ if, for all $\mathscr{R}$-instances $I$ with $D\subseteq I$, if $I\models\sigma$ for all sentences $\sigma\in\Sigma$ then $I\models{q}$, where the satisfaction relation $\models$ is defined in the standard way. 

\section{Ontologies and Languages in OMQA}

Before identifying the existence of universal languages for OMQA, we need some notions to clarify what an ontology in OMQA is, and what an ontology language in OMQA is. To make the presented results more applicable, we will define these notions in a language-independent way.

To define ontologies in OMQA, below we generalize the notion introduced in~\cite{ZhangZY16} from CQs to more general query languages such as UCQs.
\begin{defn}
Let $\mathscr{D}$ and $\mathscr{Q}$ be relational schemas, and $\mathcal{Q}$ a class of queries. A {\em quasi-OMQA[$\mathcal{Q}$]-ontology} over $(\mathscr{D},\mathscr{Q})$ is a set of ordered pairs $(D,{q})$, where $D$ is a $\mathscr{D}$-database and ${q}$ a boolean $\mathscr{Q}$-query in $\mathcal{Q}$ such that $\textit{const}(q)\subseteq{adom}(D)$.

Moreover, a quasi-OMQA[$\mathcal{Q}$]-ontology $O$ over $(\mathscr{D},\mathscr{Q})$ is called an {\em OMQA[$\mathcal{Q}$]-ontology} if all of the following hold: 
\begin{enumerate}
\item $O$ is {\em closed under query conjunctions}, i.e.,\\ $(D,\hspace{-.01cm}{q})\hspace{-.01cm}\in\hspace{-.01cm} O\,\&\,(D,\hspace{-.01cm}{p})\hspace{-.01cm}\in\hspace{-.01cm} O\,\&\,{q}\wedge{p}\hspace{-.01cm}\in\hspace{-.01cm}\mathcal{Q}\Longrightarrow(D,\hspace{-.01cm}{q}\wedge{p})\hspace{-.01cm}\in\hspace{-.01cm} O$.
\item $O$ is {\em closed under query implications}, i.e.,\\ $(D,{q})\in O\,\&\,{p}\in\mathcal{Q}\,\&\,{q}\vDash p\Longrightarrow(D,p)\in O$. 
\item $O$ is {\em closed\hspace{-.01cm} under\hspace{-.01cm} injective\hspace{-.01cm} database\hspace{-.01cm} homomorphisms},\hspace{-.01cm} i.e.,\!\\ $(D,{q})\in O\,\&\,D\twoheadrightarrow_{\textit{const(q)}} D'\Longrightarrow(D',{q})\in O$. 
\end{enumerate}
\end{defn}

Given any logical theory $\Sigma$, we can interpret it as a quasi-OMQA[$\mathcal{Q}$]-ontology over $(\mathscr{D},\mathscr{Q})$ as follows:
\begin{equation*}
[\![\Sigma]\!]_{\mathscr{D},\mathscr{Q}}^{\mathcal{Q}}=\{(D,{q}):D\!\in\!\text{DB}[\mathscr{D}]\,\&\,q\!\in\!\mathcal{Q}[\mathscr{Q}]\,\&\,D\cup\Sigma\vDash{q}\}.
\end{equation*}
It is easy to see that, for theories $\Sigma$ in almost all the classical logic, $[\![\Sigma]\!]_{\mathscr{D},\mathscr{Q}}^\mathcal{Q}$ is indeed an OMQA$[\mathcal{Q}]$-ontology.

With the notion of ontology, we are then able to present an abstract definition for ontology languages in OMQA. 
\begin{defn}
Let $V$ be a finite but nonempty set, $\mathscr{D}$ and $\mathscr{Q}$ relational schemas, and $\mathcal{Q}$ a class of queries. Then every {\em OMQA[$\mathcal{Q}$]-language} $\mathcal{L}$ over $(\mathscr{D},\mathscr{Q})$ (with vocabulary $V$) is defined as an ordered pair $(T, M)$ such that:
\begin{enumerate}
\item ${T}$ consists of a decidable set of {\em theories}, each of which is a finite string over $V$ (i.e., an element of $V^\ast$);
\item ${M}$ is a {\em semantic mapping}, i.e., a function that maps each theory in ${T}$ to an OMQA[$\mathcal{Q}$]-ontology over $(\mathscr{D},\mathscr{Q})$.
\end{enumerate}
\end{defn}

\begin{exm}\label{exm:ded}
Let $\mathscr{D}$ and $\mathscr{Q}$ be relational schemas, $\mathcal{Q}$ a class of queries, and $T$ a decidable class of finite sets of DEDs. Let $M$ be a function that maps each set $\Sigma\in T$ to $[\![\Sigma]\!]_{\mathscr{D},\mathscr{Q}}^{\mathcal{Q}}$.
It is easy to see that $\mathcal{L}=(T,M)$ is an OMQA[$\mathcal{Q}$]-language.  
\end{exm}

The language $\mathcal{L}$ defined above is called a {\em DED-language} over $(\mathscr{D},\mathscr{Q})$ (induced by $T$). In particular, if $T$ consists of all finite sets of DEDs, we call it the {\em full DED-language} over $(\mathscr{D},\mathscr{Q})$. Unfortunately, it had been proved in~\cite{Vardi1982} that query answering with the full DED-language is uncomputable. In this work, we thus focus on tractable OMQA-languages. We will consider two kinds of tractability: 

\begin{defn}\label{defn:tractability}
Let $\mathscr{D}$ and $\mathscr{Q}$ be relational schemas, $\mathcal{C}$ and $\mathcal{Q}$ classes of queries, and $\mathcal{K}$ a complexity class. An OMQA[$\mathcal{Q}$]-language $\mathcal{L}=({T},{M})$ over $(\mathscr{D},\mathscr{Q})$ is 
\begin{enumerate}
\item {\em $\mathcal{C}$-rewritable} if there is a computable function $rew$ that maps each ordered pair $(t,{q})\in T\times\mathcal{Q}[\mathscr{Q}]$ to a boolean query $\varphi_{t,{q}}\in\mathcal{C}[\mathscr{D}]$ such that $(D,{q})\in{M}(t)$ iff $D\models\varphi_{t,{q}}$; in this case, $rew$ is called a {\em $\mathcal{C}$-rewriting function} of $\mathcal{L}$.
\item {\em $\mathcal{K}$-compilable} if there is a computable function $com$ that maps each ordered pair $(t,{q})\in T\times\mathcal{Q}[\mathscr{Q}]$ to a Turing machine $\mathds{M}_{t,{q}}$, whose running time belongs to $\mathcal{K}$, such that $(D,{q})\in{M}(t)$ iff $\mathds{M}_{t,{q}}$ accepts on the input $D$; in this case, $com$ is called a {\em $\mathcal{K}$-compiler} of $\mathcal{L}$.
\end{enumerate}
\end{defn}

\begin{exm}
According to~\cite{CaliGL12}, the language of linear TGDs is both FO-rewritable and AC$^0$-compilable, and the language of guarded TGDs is both Datalog$^\neg(\le)$-rewritable and PTIME-compilable.
\end{exm}

\begin{rem}
Clearly, there is a nonuniform way to redefine notions in Definition~\ref{defn:tractability} by allowing rewriting functions and compilers to be uncomputable. However, it is worth noting that languages defined in such a way could be intractable. In fact, there is a nonuniform FO-rewritable OMQA-language in which the query answering is highly undecidable.  
\end{rem}

Next we give the definition of universal OMQA-language. 

\begin{defn}
Let $\mathcal{Q}$ be a class of queries, $\mathscr{D}$ and $\mathscr{Q}$ relational schemas, and $\mathcal{L}=({T},{M})$ and $\mathcal{L}'=({T}',{M}')$ OMQA[$\mathcal{Q}$]-languages over $(\mathscr{D},\mathscr{Q})$. Then we say that $\mathcal{L}'$ is {\em at least as expressive as} $\mathcal{L}$, written $\mathcal{L}\le\mathcal{L}'$, if for each theory $t\in T$ there is a theory $t'\in T'$ such that $M(t)=M'(t)$; and $\mathcal{L}'$ {\em has the same expressiveness as} $\mathcal{L}$ if both $\mathcal{L}\le\mathcal{L}'$ and $\mathcal{L}'\le\mathcal{L}$.

An OMQA[$\mathcal{Q}$]-language $\mathcal{L}$ is called {\em universal} for a family $\mathscr{L}$ of  OMQA[$\mathcal{Q}$]-languages over $(\mathscr{D},\mathscr{Q})$ if (i) $\mathcal{L}\in\mathscr{L}$, and (ii) for all languages $\mathcal{L}'\in\mathscr{L}$, we have that $\mathcal{L}'\le\mathcal{L}$. 
\end{defn}

\section{Nonexistence for the General Case}

One ambitious goal in OMQA is to find some universal language for the tractable OMQA. Unfortunately, the following theorem shows that this goal is in general unachievable.

\begin{thm}\label{thm:no_forewritable_univ_lang}
Let $\mathscr{D}$ and $\mathscr{Q}$ be relational schemas such that $\mathscr{Q}$ contains a relation symbol of arity $\ge 2$, and suppose $\mathcal{C}\in\{\textit{FO},\textit{FO}\,(+,\times),\textit{Datalog}^\neg(\le)\}$ and ACQ $\subseteq$ $\mathcal{Q}\subseteq\text{UCQ}$. Then there is no universal language for the family of $\mathcal{C}$-rewritable OMQA[$\mathcal{Q}$]-languages over $(\mathscr{D},\mathscr{Q})$.
\end{thm}

Since AC$^0$ and PTIME are exactly captured by FO$(+,\times)$ and Datalog$^\neg(\le)$ respectively, by Theorem~\ref{thm:no_forewritable_univ_lang} we have

\begin{cor}
Let $\mathscr{D}$ and $\mathscr{Q}$ be relational schemas such that $\mathscr{Q}$ contains at least one relation symbol of arity $\ge 2$, and suppose $\mathcal{K}\!\in\!\{$AC$^0$, PTIME$\}$ and ACQ $\subseteq$ $\mathcal{Q}\subseteq\text{UCQ}$. Then there is no universal language for the family of $\mathcal{K}$-compilable OMQA[$\mathcal{Q}$]-languages over $(\mathscr{D},\mathscr{Q})$.
\end{cor}

To prove Theorem~\ref{thm:no_forewritable_univ_lang}, the general idea is to implement a diagonalization argument as follows. Assume by contradiction that there is a universal language for the desired family. We first give an effective enumeration for all nontrivial ontologies defined in the universal language. With this enumeration, we then construct a new OMQA[$\mathcal{Q}$]-ontology $O$ and a new language $\mathcal{L}'$ in which $O$ is definable; Finally we show that $\mathcal{L}'$ is still $\mathcal{C}$-rewritable, which leads to a contradiction.  

\begin{proof}[Proof of Theorem~\ref{thm:no_forewritable_univ_lang}]
Only consider the case where $\mathcal{C}=\text{FO}$ and $\mathcal{Q}=\text{UCQ}$. Assume by contradiction that there is a universal language for FO-rewritable OMQA[UCQ]-languages over $(\mathscr{D},\mathscr{Q})$. Let $\mathcal{L}=(T,M)$ be such a language. Our task is to define another FO-rewritable OMQA[UCQ]-language that is strictly more expressive than $\mathcal{L}$. To do this, we first construct an ontology that is not definable in $\mathcal{L}$.

Before we present the construction, some notations are needed. W.l.o.g., we assume that there is a binary relation symbol $R$ in $\mathscr{Q}$. Note that, by a repetition of the parameters, $R$ can be always simulated by another relation symbol of arity $>2$. For example, one can use $S(x,x,y)$ to simulate $R(x,y)$. With this assumption, we first define a sequence of acyclic CQs. For all integers $n\ge 1$, we define
\begin{equation*}
q_n= \exists x_0\cdots\exists x_n(R(x_0,x_1)\wedge\cdots\wedge R(x_{n-1},x_n)).
\end{equation*}
Intuitively, $q_n$ asserts that there is a cycle-free path (via $R$) of length $n+1$ in the intended model.

Let $\vec{s}=(s_1,s_2,\dots)$ be an effective enumeration\footnote{I.e., there is a Turing machine to generate such an enumeration.} of all the theories in $T$. Such an enumeration clearly exists. Now our task is to construct countably infinite sequences $\vec{N}$ and $\vec{t}$, where $\vec{N}=(N_1,N_2,\dots)$ is a sequence of positive integers, and $\vec{t}=(t_1,t_2,\dots)$ is a sequence of theories in $T$. The sequences are required to have the following properties:
\begin{enumerate}
\item $\vec{N}$ is monotonic increasing, i.e., $N_i\le N_j$ if $i<j$;
\item For all $k\ge 1$ there exists a database $D$ with $(D,q_k)\in M(t_k)$ and $|{adom}(D)|\le N_k$;
\item For all $t\in T\setminus \vec{t}$ there exists $i>0$ such that $|adom(D)|>N_i+1$ for all databases $D$ with $(D,q_i)\in M(t)$. 
\end{enumerate}

Procedure~\ref{algo:enumeration} is devoted to generate the desired sequences. 
\begin{algorithm}
{
	$n\leftarrow 1$ \; 
	\For{$i\leftarrow 1$ \KwTo $\infty$}
	{
		\For{$n\leftarrow n$ \KwTo $\infty$} 
		{  
	        		\For{$j\leftarrow 1$ \KwTo $n$}
	        		{
	        			\lIf{{$\exists D\text{ s.t. }(D,q_i)\!\in\! M(s_j)$ $\,\,\&\,\,|adom(D)|\le n$}}
	        			{
	        				{\bf goto}~line 9
	        			}
				\If{{$\exists D\text{ s.t. }(D,q_i)\!\in\! M(s_j)$ $\,\,\&\,\,|adom(D)|\le n+1$}}{
	        				$n\leftarrow n+1$ \;
	        				{\bf goto} line 9 \;
				} 
			}
		}
		$N_i\longleftarrow n$ \; 
		$t_i\longleftarrow s_j$ \;
		delete $s_j$ from $\vec{s}$ \;
	}  
	\caption{Generating Sequences $\vec{t}$ and $\vec{N}$}
	\label{algo:enumeration}
}
\end{algorithm}

Now we have the following property:

\medskip
\noindent{\em Claim 1. The sequences $\vec{N}$ and $\vec{t}$ generated by Procedure~\ref{algo:enumeration} satisfy Properties (1-3).}

\begin{proof}
Properties 1 and 2 are clear from Procedure 1. So it remains to show Property 3.
Suppose $t=s_k$ for some $k\ge 1$. Since $t$ has no occurrence in $\vec{t}$, according to Procedure 1, we know that, whenever lines 5 and 6 are executed  for $j=k$, conditions in both ``if" statements must be false. (Otherwise we will have $t\in\vec{t}$.) In addition, as $n$ increases arbitrarily, we know that line 6 must be executed. This means that there is some  $i\ge 1$ such that $|adom(D)|>N_i+1$ for all databases $D$ with $(D,q_i)\in M(t)$, which then yields the claim.  
\end{proof}

Now we are able to construct the desired ontology. To do this, we first define some notations. For $n\ge 1$, let $\lambda_n$ denote the sentence $\exists x_1\cdots x_n\bigwedge_{1\le i<j\le n}\neg(x_i=x_j)$, which asserts that the intended domain contains at least $n$ elements. Given a boolean UCQ $q$, if there exists an integer $k\ge 1$ such that $q_k\vDash q$, let $\varphi_q$ denote $\lambda_{N_m+1}$ where $m$ is the least integer among such $k$s, and let $\varphi_q$ denote the sentence $\exists x\neg(x=x)$ (always false) if no such $k$s exist. Furthermore, we define
\begin{equation*}
O=\left\{(D,{q}):D\in\text{DB}[\mathscr{D}]\,\&\,q\in\text{UCQ}[\mathscr{Q}]\,\&\,D\models\varphi_q\right\}\hspace{-.03cm}.
\end{equation*}

It is not difficult to prove the following properties:

\medskip
\noindent{\em Claim 2. Let $p$ and $p'$ be boolean UCQs. Then we have:
\begin{enumerate}
\item If $p\vDash p'$ then $\varphi_p\vDash\varphi_{p'}$;
\item $\varphi_{p\wedge p'}\equiv \varphi_p\wedge\varphi_{p'}$
\end{enumerate}
}

\begin{proof}
1. For the case where there exists no integer $i\ge 1$ such that $q_i\vDash p$, we have that $\varphi_p=\exists x\neg (x=x)$, which is always unsatisfiable. This implies that $\varphi_p\vDash\varphi_{p'}$ trivially. 

Now it remains to show the case where there exists $i\ge 1$ such that $q_i\vDash p$. Let $m$ be the least integer such that $q_m\vDash p$. Then we have $\varphi_p=\lambda_{N_m+1}$. From $p\vDash p'$, we then have $q_m\vDash p'$. Let $n$ be the least integer such that $q_n\vDash p'$. Then it is clear that $n\le m$. According to Property 1, we also know that $N_n\le N_m$, which implies that $\lambda_{N_m+1}\vDash\lambda_{N_n+1}$, or equivalently, $\varphi_{p}\vDash\varphi_{p'}$. This proves Statement 1.

\smallskip
2. For the case where there is no integer $i\ge 1$ such that $q_i\vDash p$, we have that
$\varphi_{p}=\exists x\neg(x=x)=\varphi_{p\wedge p'}$, 
which implies the desired equivalence. The same argument applies to the case where there is no integer $i\ge 1$ such that $q_i\vDash p'$.

Now, it remains to consider the case where there are integers $i$ and $j$ such that $q_i\vDash p$ and $q_j\vDash p'$. Let $m$ and $n$ denote the least integers among such $i$s and $j$s, respectively. W.l.o.g., suppose $m\ge n$. Then we have both $q_m\vDash q_n\vDash p'$ and $q_m\vDash p$. Combining both of them, we know that $m$ is the least integer such that $q_m\vDash p\wedge p'$. Thus, we have that $\varphi_p=\lambda_{N_m+1}=\varphi_{p\wedge p'}$. On the other hand, it is also clear that $\lambda_{N_m+1}\vDash\lambda_{N_n+1}$, or equivalently $\varphi_p\vDash\varphi_{p'}$, which implies that $\varphi_p\wedge\varphi_{p'}\equiv\varphi_p$. Consequently, we obtain that $\varphi_{p\wedge p'}\equiv \varphi_p\wedge\varphi_{p'}$, which completes the proof.
\end{proof}

\medskip
\noindent{\em Claim 3. $O$ is an OMQA[UCQ]-ontology.}

\begin{proof}
The closure property of $O$ under injective database homomorphisms is clear since, for any boolean UCQ $q$, $\varphi_q$ is preserved under injective homomorphisms.

Next we show that $O$ is closed under query conjunctions. Suppose $(D,p)\in O$ and $(D,p')\in O$. By definition, we have both $D\models\varphi_p$ and $D\models\varphi_{p'}$, which means that $D\models\varphi_p\wedge\varphi_{p'}$. By Statement 2 of Claim 2, we then have that $D\models\varphi_{p\wedge p'}$, which implies that $(D,p\wedge p')\in O$ as desired.

Now it remains to show the closure of $O$ under query implications. Suppose $(D,p)\in O$ and $p\vDash p'$. We need to prove $(D,p')\in O$. From $(D,p)\in O$ we have $D\models\varphi_p$, and from $p\vDash p'$, we have $\varphi_p\vDash\varphi_{p'}$ by Statement 1 of Claim 2. Combining both of these, we obtain $D\models\varphi_{p'}$. By definition, it follows that $(D,p')\in O$, which completes the proof. 
\end{proof}

\medskip
\noindent{\em Claim 4. ${O}\ne M(t)$ for any theory $t\in T$.}
\medskip

\begin{proof}
First consider the case where $t$ occurs in $\vec{t}$, and suppose $t=t_k$ for some $k\ge 1$. 
According to the definition of $\vec{t}$, we know that there is a database $D$ with $(D,q_k)\in M(t_k)$ and $|adom(D)|\le N_k$. On the other hand, by the definition of $\varphi_{q_k}$ has no model $D$ with $|adom(D)|\le N_k$, which means that there is no database $D$ with $(D,q_k)\in O$ and $|adom(D)|\le N_k$. Consequently, we have $O\ne M(t)$.

Now it remains to consider the case where $t$ does not occur in $\vec{t}$. By Claim 1, it suffices to show that for every integer $k\ge 1$ there exists a database $D$ with $(D,q_k)\in O$ and $|adom(D)|\le N_k+1$, or equivalently, $\varphi_{q_k}$ has a model that contains at most $N_k+1$ elements. According to the definition of $\varphi_{q_k}$, the latter is indeed true. This also implies that $O\ne M(t)$, which completes the proof immediately.
\end{proof}

With Claims~3 and 4, we are now in the position to prove the desired theorem. Let $t'$ be a binary string such that $t'\not\in T$, and let $T'=T\cup\{t'\}$. Following the decidability of $T$, we have the decidability of $T'$. Let $M'$ be a function that extends $M$ by mapping $t'$ to $O$, and let $\mathcal{L}'=(T',M')$. By Claim~3, we know that $\mathcal{L}'$ is an OMQA[UCQ]-language. Suppose $rew$ is an FO-rewriting function. Let $rew'$ be a function that extends $rew$ by mapping $(t',q)$ to $\varphi_q$ for all boolean UCQs $q$. By a slight modification to Procedure~\ref{algo:enumeration}, one can easily devise an algorithm to compute $N_i$ (and $s_i$) on given integer $i\ge 1$. This implies that $rew'$ is computable. By definition, we know that $rew$ is an FO-rewriting function, which implies that $\mathcal{L}'$ is FO-rewritable. By Claim~4, we also know that $\mathcal{L}'$ is strictly more expressive than $\mathcal{L}$, a contradiction as desired. And this completes the proof. 
\end{proof}

\begin{rem}
Since the sentence $\varphi_q$ defined in the above proof is also an FO$(+,\times)$-sentence, so the proof directly applies to the case of FO$(+,\times)$. For a proof of the remaining case, one can convert $\varphi_q$ to a Datalog$^\neg(\le)$-program.
\end{rem}

\section{Locality to the Rescue}

In the last section, we proved that there is no universal language for tractable OMQA in general. Then, a natural question arises as to whether one can find a natural property that approximates the tractability but still allows the existence of a universal language. The challenge here is that the property should be manageable enough to avoid a diagonalization argument (see the proof of Theorem~\ref{thm:no_forewritable_univ_lang}). Below we propose a property as an approximation of the FO-rewritability. 

\subsection{Locality as Approximation of FO-rewritability}

A {\em bound function} is a computable function $\ell:\mathbb{N}\rightarrow\mathbb{N}$ such that $\ell(n)\ge n$ for $n\in\mathbb{N}$. To simplify the presentation, we fix a way to represent bound functions, e.g., one can represent each bound function by a Turing machine that computes it. A class of bound functions is called {\em decidable} if the class of representations of those bound functions is decidable.

To measure the size of a query, we fix a computable function $\dl\cdot\dl$ that maps each UCQ to a positive integer. Clearly, there are many methods to define $\dl\cdot\dl$. The only restriction is that we require $\dl p\wedge q\dl\ge\dl p\dl+\dl q\dl$ for all UCQs $p$ and $q$.

\begin{defn}
Let $\mathscr{D}$ and $\mathscr{Q}$ be relational schemas, and $\mathcal{Q}$ a class of queries, $O$ an OMQA[$\mathcal{Q}$]-ontology over $(\mathscr{D},\mathscr{Q})$, and $\ell$ a bound function. Then $O$ is called {\em$\ell$-local} if for all boolean $\mathscr{Q}$-queries ${q}\in\mathcal{Q}$ and all $\mathscr{D}$-databases $D$ there is a set $A$, which consists of at most $\ell(\dl {q}\dl )$ constants, such that 
\begin{equation*}
(D,{q})\in O\quad\text{ iff }\quad(D|_{A},{q})\in O.
\end{equation*}

Furthermore, given an OMQA[$\mathcal{Q}$]-language $\mathcal{L}$, a bound function $\ell$ and a class $\mathsf{F}$ of bound functions, $\mathcal{L}$ is called {\em $\ell$-local} if all-OMQA[$\mathcal{Q}$] ontologies defined in $\mathcal{L}$ is $\ell$-local, and $\mathcal{L}$ is $\mathsf{F}$-local if it is $\ell'$-local for some bound function $\ell'\in\mathsf{F}$. 
\end{defn}

One might question why the bounded locality is a good approximation to the first-order rewritability. Let $\exists^+\text{FO}(\ne)$ denote the class of first-order sentences built on atoms and inequalities by using connectives $\wedge,\vee$ and the quantifier $\exists$. Obviously, this class is exactly the class of UCQs with inequalities. It had been observed by~\citeauthor{Benedikt16} that $\exists^+\text{FO}(\ne)$ captures the class of first-order sentences that preserved under injective homomorphisms~\shortcite{Benedikt16}. It remains open whether such a preservation theorem holds on finite structures (or databases). If this is indeed true, by the following proposition we then have that an OMQA-language is FO-rewritable iff it is $\ell$-local for some bound function $\ell$.

\begin{prop}\label{prop:locality2rewritability}
Let $\mathscr{D}$ and $\mathscr{Q}$ be relational schemas, $O$ an OMQA[UCQ]-ontology over $(\mathscr{D},\mathscr{Q})$, and $\ell$ a bound function. Then $O$ is $\ell$-local iff for each boolean $\mathscr{Q}$-UCQ $q$ there is a $\exists^+\text{FO}(\ne)$-sentence $\varphi$ involving at most $\ell(\dl {q}\dl )$ terms such that $(D,q)\in O$ iff $D\models\varphi$ for all $\mathscr{D}$-databases $D$.
\end{prop}

\begin{proof} 
For the direction of ``if", let us assume that for each boolean $\mathscr{Q}$-UCQ $q$ there is a $\exists^+\text{FO}(\ne)$-sentence $\varphi$ involving at most $\ell(\dl {q}\dl )$ terms such that $(D,q)\in O$ iff $D\models\varphi$ for all $\mathscr{D}$-databases $D$. We need to show that $O$ is $\ell$-local. Let $q$ be a boolean $\mathscr{Q}$-UCQ, and $\varphi$ a $\exists^+\text{FO}(\ne)$-sentence involving at most $\ell(\dl {q}\dl )$ terms such that $(D,q)\in O$ iff $D\models\varphi$ for all $\mathscr{D}$-databases $D$. By the assumption, such a sentence does exist. Suppose $\varphi=\exists \vec{x}\psi$ where $\psi$ is a quantifier-free existential positive first-order formula with inequalities and involving at most $\ell(\dl q\dl)$ terms. Let $D$ be a $\mathscr{D}$-database. If $D\models\varphi$, then let $A=\{s(x):x\in\vec{x}\}\cup const(\varphi)$, where $s$ is an assignment such that $D\models\psi[s]$. Otherwise, let $A$ be any subset of $adom(D)$ such that $|A|\le k$. In both cases we have the following: (i) $|A|\le k$, and (ii) $D\models\varphi$ iff $D|_A\models\varphi$. From the latter, we know that $(D,q)\in O$ iff $(D|_A,q)\in O$. We thus yields that $O$ is $\ell$-local as desired.

Conversely, suppose $O$ is $\ell$-local. Let $q$ be a boolean $\mathscr{Q}$-UCQ. Given a $\mathscr{D}$-database $D$, let $A\subseteq adom(D)$ be a witness of the locality of $O$ w.r.t. $q$, let $\varphi_D$ denote the conjunction of all facts in $D$; let $\hat{\varphi}_D$ be the formula obtained from $\varphi_D$ by replacing each constant that does not occur in ${q}$ by a fresh variable; and let $\psi_D$ denote the sentence $\exists\vec{x}(\hat{\varphi}_D\wedge\lambda_{\vec{x}})$, where $\vec{x}$ is the tuple of all variables occurring in $\hat{\varphi}_D$, and $\lambda_{\vec{x}}$ denotes the conjunction of $x_i\ne x_j$ for each pair of distinct variables $x_i,x_j\in\vec{x}$. It is easy to see that, up to logical equivalence, there is only a finite number of $\psi_D$ for all $\mathscr{D}$-databases $D$ such that $(D,q)\in O$. Let $\psi_{q}$ be a disjunction of $\psi_D$ for all $\mathscr{D}$-databases $D$ with $|adom(D)|\le\ell(\dl {q}\dl )$ and $(D,q)\in O$. Clearly, $\psi_q$ is equivalent to a $\exists^+\text{FO}(\ne)$-sentence that involves at most $\ell(\dl {q}\dl )$ terms. To complete the proof, it suffices to show the following property:

\medskip
{\noindent\em Claim.} $(D,{q})\in O$ iff $D\models\psi_{q}$ for all databases $D$ over $\mathscr{D}$.
\medskip

Now it remains to prove the claim. Let $D$ be a database over $\mathscr{D}$. We first prove the direction of ``only if". Suppose $(D,{q})\in O$. Since $O$ is $\ell$-local, there should be a $\mathscr{D}$-database $D'\subseteq D$ such that $(D',q)\in O$ and $|adom(D')|\le\ell(\dl {q}\dl )$. By the definition of $\psi_{q}$, we know that $\psi_{D'}$ is equivalent to a disjunct of $\psi_{q}$. It is clear that $D'\models\psi_{D'}$. From $D'\subseteq D$ we the have $D\models\psi_{D'}$, which implies $D\models\psi_{q}$ as desired.

For the converse, we assume that $D\models\psi_{q}$. Then there is a database $D'$ over $\mathscr{D}$ with $|adom(D')|\le\ell(\dl {q}\dl )$ such that (i) $(D',q)\in O$ and (ii) $\psi_{D'}$ is a disjunct of $\psi_q$. From (ii) we have $D\models\psi_{D'}$, which means that there is an injective $C$-homomorphism from $D'$ to $D$, where $C=const(q)$.  As $O$ is closed under injective database homomorphisms, we have $(D,{q})\in O$ as desired, which completes the proof.
\end{proof}

\begin{rem}
Proposition~\ref{prop:locality2rewritability} reveals an intrinsic connection between the bounded locality and the complexity of rewritings. We will elaborate this in an extended version of this paper.
\end{rem}

\subsection{Universal Language for Local OMQA}

Now it remains to know whether the bounded locality allows the existence of universal languages. For convenience, in the rest of this section, we {\em fix $\mathsf{F}$ as a decidable class of bound functions; fix $\mathscr{D}$ and $\mathscr{Q}$ as a pair of disjoint relational schemas}. The disjointness will not introduce any real limitation. For instance, in a DED-language, given any set $\Sigma$ of DEDs, one can construct another set $\Sigma'$ of DEDs by introducing a fresh relation symbol $\textit{R}'$ for each $\textit{R}\in\mathscr{D}$, and adding copy rules of the form $\textit{R}'(\vec{x})\rightarrow\textit{R}(\vec{x})$. Clearly, $\Sigma'$ has the same behaviour over $(\mathscr{D}',\mathscr{Q})$ as $\Sigma$ over $(\mathscr{D},\mathscr{Q})$, where $\mathscr{D}'$ denotes the schema consisting of all the fresh symbols.

Surprisingly, we have the following result.

\begin{thm}\label{thm:univ_lang_exist}
Let $\mathcal{Q}$ be a decidable class of UCQs. Then there exists a DED-language that is universal for the family of $\mathsf{F}$-local OMQA[$\mathcal{Q}$]-languages over $(\mathscr{D},\mathscr{Q})$.
\end{thm}

Let $\ell$ be any bound function in $\mathsf{F}$. To prove Theorem~\ref{thm:univ_lang_exist}, the general idea is to develop a transformation that converts every DED set to an $\ell$-local DED set. In addition, for each DED set that is already $\ell$-local, the transformation is required to preserve the semantics of query answering. If such a transformation exists, since DED is universal for the family of OMQA-languages in which query answering is recursively enumerable, we then obtain a universal language for the family of $\mathsf{F}$-local OMQA-languages.

Let us begin with a finite set $\Sigma$ of DEDs over a relational schema $\mathscr{R}\supseteq\mathscr{D}\cup\mathscr{Q}$. To implement the desired transformation, we first show how to construct an $\ell$-local OMQA[$\mathcal{Q}$]-ontology from $\Sigma$. As a natural idea, one may expect to define the desired ontology by removing all the pairs $(D,q)$ from the original ontology (defined by $\Sigma$) where $q$ is not $\ell$-local on $D$. Unfortunately, the ontology defined above is in general not well-defined. To construct the desired ontology, the $\ell$-locality and the closure under both query conjunctions and query implications should be maintained simultaneously. 

Below we explain how to construct the ontology. We need to fix a strict linear order $\prec$ over $\mathscr{Q}$-UCQs firstly. The strict linear order is required to satisfy $p\prec q$ for all $\mathscr{Q}$-UCQs $p$ and $q$ such that $\dl p\dl<\dl q\dl$. Clearly, such an order always exists. For the given set $\Sigma$ of DEDs, let $S_\Sigma$ be the set that consists of the ordered pair $(D,q)$ if $D$ is a $\mathscr{D}$-database, $q$ is a $\mathscr{Q}$-UCQ in $\mathcal{Q}$, and the following condition holds:
\begin{equation}\label{eqn:condition}
\begin{aligned}
\hspace{-.16cm}\forall p&\hspace{-.03cm}\in\hspace{-.03cm}\mathcal{Q}[\mathscr{Q}]:\dl p\dl\le\dl\widehat{pr}(q)\dl\,\&\,\widehat{pr}(q)\vDash p\\
&\!\Longrightarrow\!
\exists A\!\subseteq\!adom(D)\text{ s.t. }|A|\!\le\!\ell(\dl p\dl)\,\&\,D|_A\!\cup\hspace{-.05cm}\Sigma\vDash p
\end{aligned}
\end{equation}
where $pr(q)$ is the set of boolean UCQs $p$ such that $p\prec q$ and $(D,p)\in S_\Sigma$, and $\widehat{pr}(q)$ denotes the conjunction of $q$ and all UCQs in $pr(q)$. Moreover, we define $O_\Sigma$ as the minimum superset of $S_\Sigma$ that is closed under query conjunctions, query implications and injective database homomorphisms.

The constructed ontology enjoys several properties which will play important roles in our proof for Theorem~\ref{thm:univ_lang_exist}.

\begin{lem}\label{lem:correctness}
If $\db{\Sigma}^\mathcal{Q}_{\mathscr{D},\mathscr{Q}}$ is $\ell$-local, then $O_\Sigma=\db{\Sigma}^\mathcal{Q}_{\mathscr{D},\mathscr{Q}}$.
\end{lem}

\begin{proof}
Follows from the definition of $O_\Sigma$.
\end{proof}

\begin{lem}\label{lem:locality}
$O_\Sigma$ is an $\ell$-local OMQA[$\mathcal{Q}$]-ontology.
\end{lem}

\begin{proof}
Let $q$ be a UCQ $q$ with $(D,q)\in O_\Sigma$. We first claim that there is an integer $k\ge 1$ and UCQs $p_1,\dots,p_k$ such that $(D,p_i)\in S_\Sigma$ for each $p_i$ and $p_1\wedge\cdots\wedge p_k\vDash q$. This can be proved by a routine induction on the construction of $O_\Sigma$.

With the claim, w.l.o.g., we assume $p_1\prec p_2\prec\cdots\prec p_k$. Let $p$ denote the query $p_1\wedge\cdots\wedge p_k$. Obviously, it holds that $\widehat{pr}(p_k)\vDash p$, which implies that $\widehat{pr}(p_k)\vDash q$ immediately. On the other hand, by definition we know that $p_i$ is $\ell$-local on $D$, i.e., there is a set $A_i\subseteq adom(D)$ such that $D|_{A_i}\cup\Sigma\vDash p_i$. Let $A=A_1\cup\cdots\cup A_k$. We then have $D|_{A}\cup\Sigma\vDash p$ and 
$|A|\le\ell(\dl p_1\dl)+\cdots+\ell(\dl p_k\dl)\le\ell(\dl p_1\dl+\cdots+\dl p_k\dl)\le\ell(\dl p\dl)$.
Consequently, we obtain that $p$ is $\ell$-local on $D$. 

Now, it remains to show that $q$ is $\ell$-local on $D$. For the case where $\dl p\dl\le\dl q\dl$, from $p\vDash q$ and $D|_A\cup\Sigma\vDash p$, we obtain that $D|_A\cup\Sigma\vDash q$, which implies that $q$ is $\ell$-local on $D$. For the other case, it must be true that $\dl p\dl>\dl q\dl$. From the fact that $\dl\widehat{pr}(p_k)\dl\ge\dl p\dl$, we know that $\dl\widehat{pr}(p_k)\dl>\dl q\dl$. Since $(D,p_k)\in S_\Sigma$ is true, by Condition~(\ref{eqn:condition}) we know that there is a set $B\subseteq adom(D)$ such that $D|_B\cup\Sigma\vDash q$, which means that $q$ is $\ell$-local on $D$. This then completes the proof.
\end{proof}

\begin{lem}
$O_\Sigma$ is recursively enumerable. 
\end{lem}

\begin{proof}
The lemma is a corollary of the following facts:
(1) The validity problem for inference in first-order logic is recursively enumerable;
(2) The query containment problem for UCQs is decidable;
(3) There are only a finite number of boolean $\mathscr{Q}$-UCQs $p$ with $\dl p\dl\le\dl\widehat{pr}(q)\dl$; 
(4) There are only a finite number of subsets of $adom(D)$.
(5) Both $\ell$ and $\dl\cdot\dl$ are computable.
(6) $\mathcal{Q}$ is decidable.
\end{proof}

Now, to define the transformation, it remains to show how to encode the ontology $O_\Sigma$ by another set of DEDs. Suppose $D$ is the underlying database, and $q$ the underlying query. The encoding will be implemented in the following way: 
\begin{enumerate}
\item Simulate the query answering of $q$ under $O_\Sigma$ and $D$.
\item If the answer of Stage 1 is positive, then nondeterministically copy disjuncts of $q$ to generate the universal models. 
\end{enumerate}

The main challenges of implementing the above encoding are as follows. Firstly, instead of a single universal model, we need to generate a set of universal models in Stage 2. It is not clear whether the technique of generating universal model in~\cite{ZhangZY16} can be applied to this situation. Secondly, to encode the computation in Stage 1, a successor relation is needed. But it seems  impossible to define such a relation in the language of DEDs explicitly. 

Below we explain how to implement the encoding.

\medskip
{\noindent\bf Defining Successor and Arithmetic Relations.\,\,} 
To implement the desired encoding, a successor relation needs to be defined so that the constants in the underlying database $D$ can be ranged over. 
As there is no negation in the body of DEDs, it seems impossible to construct DEDs to traverse ALL constants in $adom(D)$. Fortunately, thanks to the closure of OMQA-ontologies under injective database homomorphisms, we do not need a successor relation on the full domain. The reason is as follows. Suppose we want to show that a query $q$ is derivable from the database $D$ under some ontology $O$. As $O$ is closed under injective database homomorphisms,  it is equivalent to show {\em whether there is a subset $A$ of $adom(D)$ such that $q$ is derivable from $D|_A$ under $O$.} 

To range over subsets of $adom(D)$, we employ the {\em partial successor relations} on $adom(D)$, each of which is a successor relation on some subset of $adom(D)$. Clearly, there is a partial successor relation for each subset $A$ of $adom(D)$. With the mentioned property, we will define some DEDs to generate partial successor relations on $adom(D)$. To check whether $q$ is derivable from $D$ under $O$, it would be sufficient to test whether the computation of Stage 1 halts with ``accept" under a certain partial successor relation.

Our method to generate partial successor relations was inspired by~\citeauthor{RudolphT2015}'s technique to define successor relations in the language of TGDs~\shortcite{RudolphT2015}. In that paper they showed that every homomorphism-closed database query can be defined by a set of TGDs. It is worth noting that the ontology mediated queries focused on this paper are not necessary to be closed under homomorphisms. So their technique cannot be applied directly. 
%
%
Fortunately, a linear order on $adom(D)$ can be easily defined by a set of DEDs, and with this order, we are able to use their idea to generate all partial successor relations that are compatible with the defined order. Now we show how to implement this idea.

Let $\textit{AD}$ be a unary relation symbol that will be interpreted as $adom(D)$. Clearly, such a relation can be easily defined by some DEDs.
With the relation $\textit{AD}$, a linear order relation $\textit{Less}$ over $\textit{AD}$ can then be defined in a routine way: 
\begin{eqnarray}
\label{eq:succ_guess} \textit{AD}(x)\wedge\textit{AD}(y)\rightarrow\textit{Less}(x,y)\vee x=y\vee\textit{Less}(y,x)\\
\textit{Less}(x,y)\wedge\textit{Less}(y,z)\rightarrow\textit{Less}(x,z)\\
\label{eq:succ_undesired}\textit{Less}(x,x)\rightarrow\bot
\end{eqnarray}

To generate all partial successor relations compatible with $\textit{Less}$, we link each constant $c$ in $adom(D)$ with a {\em alias} $a$ by the relation $\textit{Link}(c,a)$. Suppose $a_1$ and $a_2$ are aliases of constants $c_1$ and $c_2$ respectively, by $\textit{Next}(a_1,a_2)$ we mean that $c_2$ is the immediate successor of $c_1$ in the underlying successor relation. The head (resp., the tail) of a partial successor relation is denoted by $\textit{First}(a)$ (resp., $\textit{Last}(b)$). In particular, we use $a$ as the {\em name} of the underlying relation. Every partial successor relation is required to have a head and a tail. To generate these relations, we use the following DEDs:
\begin{eqnarray}
\label{eqn:succ_1}\textit{AD}(x)\rightarrow\exists v\,\textit{Link}(x,v)\wedge\textit{Last}(v)\wedge\textit{First}(v)\\
\label{eqn:succ_1}\textit{AD}(x)\rightarrow\exists v\,\textit{Link}(x,v)\wedge\textit{Last}(v)\wedge\textit{Partial}(v)\\
\label{eqn:succ_3}\begin{aligned}
\textit{Less}(x,y&)\wedge\textit{Link}(y,v)\wedge\textit{Partial}(v)\\
&\rightarrow\,\exists u\,\textit{Link}(x,u)\wedge\textit{Next}(u,v)\wedge\textit{First}(u)
\end{aligned}
\\
\label{eqn:succ_4}\begin{aligned}
\textit{Less}(x,y&)\wedge\textit{Link}(y,v)\wedge\textit{Partial}(v)\\
&\rightarrow\,\exists u\,\textit{Link}(x,u)\wedge\textit{Next}(u,v)\wedge\textit{Partial}(u)
\end{aligned}\end{eqnarray}
To understand how these DEDs work in more detail, please refer to the following example.
\begin{figure}[ht]
	\centering
	\includegraphics[width=.95\columnwidth]{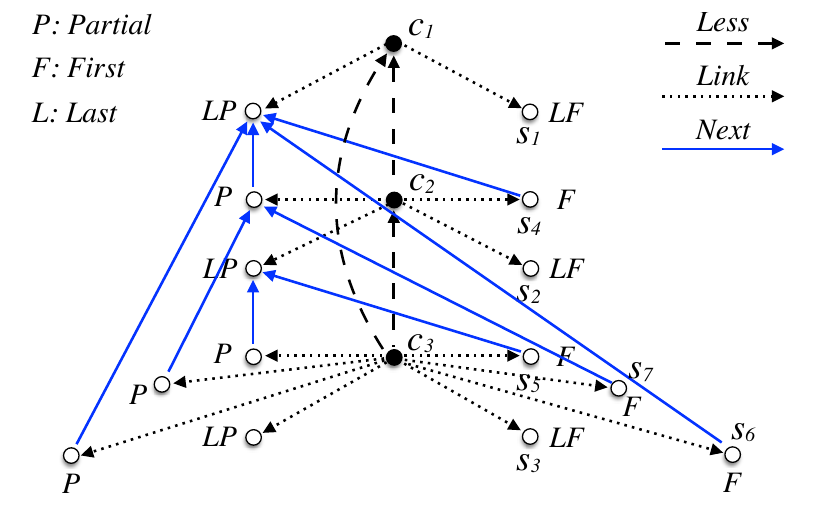}
	\caption{The Instance Generated in Example~\ref{exm:succ}.}\label{fig:succ}
\end{figure}

\begin{exm}\label{exm:succ}
Let $D$ be a database which involves only constants $c_1,c_2$ and $c_3$, and suppose the linear order defined by $\textit{Less}$ is $c_1>c_2>c_3$. By an exhaustive application of DEDs (\ref{eqn:succ_1}-\ref{eqn:succ_4}),\footnote{To make the figure simple, we use the semi-oblivious chase.} we obtain an instance as illustrated by Figure~\ref{fig:succ}.

As seen in Figure~\ref{fig:succ}, there are 7 partial successor relations $s_1,\dots,s_7$ generated in the instance. For instance, $s_1$ defines the partial successor relation involving only $c_1$; $s_5$ defines the relation $c_2>c_3$; $s_7$ defines the relation $c_1>c_2>c_3$.
\end{exm}

To encode Stage 1 mentioned before, we need to generate a linear order and the corresponding successor relation on a countably infinite domain, making sure they are compatible with the underlying partial successor relation on $adom(D)$.
More relations are needed to do this. $\textit{Zero}(o,a)$ means that $a$ is the least element under the order $o$; $\textit{DMax}(o,a)$ states that $a$ is the largest element in $adom(D)$ under the order $o$; $\textit{Succ}(o,a,b)$ denotes that $b$ is the immediate successor of $a$ under the order $o$; and $\textit{LT}(o,a,b)$ asserts that $a$ is less than $b$ under the order $o$. For a technical reason, we also need some auxiliary relations. $\textit{TC}(o,a)$ denotes that $a$ is not less than the largest element in $adom(D)$ under the order $o$, and $\textit{RT}(o,a,c)$ means that $a$ is an alias of $c$ and it is used to build the order $o$.
All of these are defined by the following DEDs:
\begin{eqnarray}
\textit{First}(w)\wedge\textit{Link}(x,w)\rightarrow\textit{Zero}(w,x)\wedge\textit{RT}(w,w,x)\\
\begin{aligned}
\textit{Next}(u,v)\wedge\textit{RT}(w,u,\hspace{0.06cm}&x)\wedge\textit{Link}(y,v)\\
&\rightarrow\textit{Succ}(w,x,y)\wedge\textit{RT}(w,v,y)
\end{aligned}\\
\textit{Last}(v)\wedge\textit{RT}(w,v,x))\rightarrow\textit{DMax}(w,x)\wedge\textit{TC}(w,x)\\
\textit{TC}(v,x)\rightarrow\exists z\,\textit{Succ}(v,x,y)\wedge\textit{TC}(v,y)\\
\textit{Succ}(v,x,y)\rightarrow\textit{LT}(v,x,y)\\
\textit{LT}(v,x,y)\wedge\textit{LT}(v,y,z)\rightarrow\textit{LT}(v,x,z)
\end{eqnarray}

With these relations, it is routine to define arithmetic relations such as $\textit{Add}(o,a,b,c)$ (asserting that $c=a+b$ under the order $o$) and $\textit{Bit}_i(o,a,b)$ (asserting that the $b$-th bit of the binary representation of $a$ is $i$). We omit the details here. 

\medskip
{\noindent\bf Simulating Query Answering under $O_\Sigma$.\,\,} With a partial successor relation and the related arithmetic relations, we are now in the position to define some DEDs to simulate the query answering of $q$ under $O_\Sigma$ and $D$.  

Our encoding that implements the simulation of query answering is almost the same as that in Section 5.3 of~\cite{ZhangZY16}. 
As proved by~\citeauthor{ZhangZY16} (see Proposition 6 of~\shortcite{ZhangZY16}), all recursively enumerable OMQA-ontologies can be recognized by a certain class of Turing machines, called convergent $2$-bounded nondeterministic Turing machines. Although the queries involved in that work are only boolean CQs, by a similar argument one can show that the result can be generalized to the case where boolean UCQs are involved. The only difference is that, to deal with UCQs, we have to change the format of input slightly. Due to the space limit, we omit the details here.

With the result mentioned above, we can then find a convergent $2$-bounded nondeterministic Turing machine $\mathds{M}_\Sigma$ to recognize $O_\Sigma$. By employing the DEDs defined in Section 5.3 of~\cite{ZhangZY16} (with a slight modification to specify the partial successor relation), we are then able to simulate the computation of $\mathds{M}_\Sigma$ on the input $(D,q)$. 

To restore the result of the query answering, we use a binary relation symbol $\textit{Accept}$. By $\textit{Accept}(o,a)$ we mean that the machine $\mathds{M}_\Sigma$ halts on input $(D,q)$ with ``accept", and $o$ is the partial successor relation to implement the computation.  

\medskip
{\noindent\bf Generating Universal Models.\,\,} By applying all the DEDs that we have constructed, we will obtain the class of boolean UCQs that are derivable from $D$ under $O_\Sigma$. With such a class of UCQs, now our task is to construct a universal model set. 

Given a class $\mathcal{D}$ of databases and a set $C$ of constants, let 
\begin{equation*}
{\bigoplus}_C\mathcal{D}=\bigcup\{D^\ast:D\in\mathcal{D}\}
\end{equation*} where, for each database $D\in\mathcal{D}$, $D^\ast$ is an isomorphic copy of $D$ such that, for any pair of distinct databases $D_1,D_2\in\mathcal{D}$, only constants from $C$ will be shared by $D^\ast_1$ and $D^\ast_2$. 

Let $D$ be a $\mathscr{D}$-database, and $O$ an OMQA[UCQ]-ontology over $(\mathscr{D},\mathscr{Q})$. Given a boolean UCQ $q$, let $\mathcal{D}_q$ denote the set consisting of $[p]$ for each disjunct (a boolean CQ) of $q$. Let 
\begin{equation*}
\Gamma(O,D)=\{\mathcal{D}_q: (D,q)\in O\}.
\end{equation*} 
Let $\mathcal{U}(O,D)$ denote the set that consists of 
$
{\bigoplus}_{C}H
$
for each minimum hitting set $H$ of $\Gamma(O,D)$, where $C=adom(D)$.

\begin{prop}\label{prop:query2um}
Let $O$ be an OMQA[UCQ]-ontology over $(\mathscr{D},\mathscr{Q})$, $D$ a $\mathscr{D}$-database, and $q$ a boolean $\mathscr{Q}$-UCQ such that $\textit{const}(q)\subseteq\textit{adom}(D)$. Then $(D,q)\in O$ iff $I\models q$ for all instances $I\in \mathcal{U}(O,D)$.
\end{prop}

\begin{proof}
Let $\Lambda$ denote the set of all boolean UCQs $p$ such that $(D,p)\in O$. We first prove a property as follows.

\medskip
{\noindent\em Claim.} $\Lambda\vDash q$ iff $I\models q$ for all instances $I\in\mathcal{U}(O,D)$.%

\begin{proof}
First consider the direction of ``only if". Suppose we have $\Lambda\vDash q$, and let $I$ be any instance in $\mathcal{U}(O,D)$. We need to prove that $I\models q$. According to the definition of $\mathcal{U}(O,D)$, we know that there is a minimum hitting set $H$ of $\Gamma(O,D)$ such that $I=\bigoplus_{\textit{adom}(D)}H$. This implies that for each UCQ $q_0\in\Lambda$ there is a disjunct $p$ (which is a boolean CQ) of $q_0$ such that $[p]$ has an isomorphic copy in $I$. Consequently, we have that $I\models q_0$ for all boolean UCQs $q_0\in\Lambda$. From the assumption that $\Lambda\vDash q$, we conclude $I\models q$ as desired.

Next let us turn to the direction of ``if". Suppose $I\models q$ for all instances $I\in\mathcal{U}(O,D)$. Now our task is to prove that $\Lambda\vDash q$. Let $J$ be an arbitrary instance such that $J\models p$ for all boolean UCQs $p\in\Lambda$. Take $p$ as any boolean UCQ in $\Lambda$. W.l.o.g., we write $p$ as the form $p_1\vee\cdots\vee p_n$ where each $p_i$ is a boolean CQ. Let $\varphi_p\in\{p_1,\dots,p_n\}$ be any disjunct of $p$ such that $J\models\varphi_p$. Such a disjunct always exists because $J\models p$. Suppose $\varphi_p$ is of the form $\exists\vec{x}_p\psi_p$, where $\psi_p$ is a conjunction of atoms and $\vec{x}_p$ a tuple of variables. Let $s_p$ be an assignment such that $J\models\psi_p[s_p]$. Let $H$ denote the set that consists of $[\varphi_p]$ for each UCQ $p\in\Lambda$, and let $I$ denote the instance $\bigoplus_{\textit{adom}(D)}H$. Let $h$ be a mapping that maps the isomorphic copy of $x$ in $I$ to $s_p(x)$ if $p\in\Lambda$ and $x\in\vec{x}_p$, and maps each constant in $\textit{adom}(D)$ to itself. Clearly, $h$ is an $\textit{adom}(D)$-homomorphism from $I$ to $J$. By the assumption made in the begin of this paragraph, we know $I\models q$. As $q$ is preserved under $\textit{adom}(D)$-homomorphisms, we conclude $J\models q$, which completes the proof of the claim.
\end{proof}

According to the above claim, to prove the desired proposition, it suffices to show that $(D,q)\in O$ iff $\Lambda\vDash q$. The direction of ``only if" is trivial since from $(D,q)\in O$ we already have $q\in\Lambda$. It thus remains to show the converse. Suppose $\Lambda\vDash q$. According to the compactness, there is a finite subset $\Lambda_0$ of $\Lambda$ such that $\Lambda_0\vDash q$. Let $q_0$ denote the conjunction of all UCQs in $\Lambda_0$. Obviously, $q_0$ is also a boolean UCQ. By the definition of $\Lambda$, we know that for each UCQ $q\in\Lambda_0$ we have $(D,q)\in O$. Since $O$ is closed under query conjunctions, we obtain $(D,q_0)\in O$. By $\Lambda_0\vDash q$, it also holds that $q_0\vDash q$. Furthermore, by applying the closure property of $O$ under query implications, we conclude $(D,q)\in O$, which completes the proof of the proposition immediately. 
\end{proof}

With Proposition~\ref{prop:query2um}, we are now able to construct a set of DEDs to generate $\mathcal{U}(O,D)$. Several relations are needed to access the encoding of a query. We use $\textit{UCQ}(o,a)$ to denote that, under the order $o$, $a$ is the representation (e.g., the G\"{o}del number) of a boolean UCQ, and use $\textit{Union}(o,a,b,c)$ to assert that, under the order $o$, the boolean UCQ $a$ is the disjunction of a boolean CQ $b$ and a boolean UCQ $c$. By $\textit{CQ}(o,a)$ we mean that $a$ is the encoding of some boolean CQ under the order $o$, and by $\textit{QVar}(o,a,b)$ we means that $b$ is a variable that occurs in the boolean CQ encoded by $a$ under the order $o$. Moreover, we assume that $\textit{Q}_1,\dots,\textit{Q}_n$ enumerates all the relation symbols in $\mathscr{Q}$. For $i=1,\dots,k$, let $\textit{HasQ}_i(o,a,\vec{t})$ denote that $\textit{Q}_i(\vec{t})$ is an atom in the CQ encoded by $a$. It is easy to see that all these relations can be defined by standard arithmetic relations. 

Now, we use the following DEDs to nondeterministically choose which disjunct of a boolean UCQ to be true:
\begin{eqnarray}
\textit{Accept}(v,x)\hspace{-.04cm}\rightarrow\hspace{-.04cm}\textit{True}(v,x)\\
\!\!\!\textit{True}(v,x)\hspace{-.04cm}\wedge\hspace{-.04cm}\textit{Union}(v,x,y,z)\hspace{-.04cm}\rightarrow\hspace{-.04cm}\textit{True}(v,y)\hspace{-.04cm}\vee\hspace{-.04cm}\textit{True}(v,z)%
\end{eqnarray}
where $\textit{True}(o,a)$ states that the boolean CQ encoded by $a$ under the order $o$ is chosen to be true in the intended model.

To generate a copy of a boolean CQ in the intended universal model, we employ the following DEDs:
\begin{eqnarray}
\textit{CQ}(v,x)\wedge\textit{QVar}(v,x,y)\rightarrow\exists z\,\textit{Copy}(v,x,y,z)\\
\textit{CQ}(v,x)\wedge\textit{AD}(y)\rightarrow\textit{Copy}(v,x,y,y)\\
\begin{aligned}
\textit{True}(v,x)\wedge\textit{CQ}(v,x&)\wedge\textit{HasQ}_i(v,x,\vec{y})\\
&\wedge\textit{Copy}(v,x,\vec{y},\vec{z})\rightarrow\textit{Q}_i(\vec{z})
\end{aligned}
\end{eqnarray}
where $\textit{Copy}(v,x,\vec{y},\vec{z})$ denotes $\bigwedge_{1\le j\le k}\textit{Copy}(v,x,y_j,z_j)$ if $\vec{y}=y_1\cdots y_k$, $\vec{z}=z_1\cdots z_k$, and $k$ is the arity of $\textit{Q}_i$. Intuitively, the first rule generates a copy for each variable in the CQ $q$, the second one asserts that the constant in $q$ will not change, and the third one then copy atoms in $q$ into the universal model and implement some necessary substitutions.   

\medskip
Let $loc^\ell(\Sigma)$ denote the set of DEDs that we have defined in this section. From the encoding and Proposition~\ref{prop:query2um} we have 
\begin{lem}\label{lem:soundness}
$[\![loc^\ell(\Sigma)]\!]_{\mathscr{D},\mathscr{Q}}^{\mathcal{Q}}=O_\Sigma$.
\end{lem}

Clearly, given the representation of any bound function $\ell$ and any finite set $\Sigma$ of DEDs, constructing $loc^\ell(\Sigma)$ is computable. Let $T_\mathsf{F}$ consist of $loc^\ell(\Sigma)$ for each finite set $\Sigma$ of DEDs and each bound function $\ell\in\mathsf{F}$. Since $\mathsf{F}$ is decidable, we know that $T_\mathsf{F}$ is also decidable. Let $\mathcal{L}_\mathsf{F}$ be the DED-language induced by $T_{\mathsf{F}}$. By Lemmas~\ref{lem:soundness},~\ref{lem:correctness} and~\ref{lem:locality}, $\mathcal{L}_\mathsf{F}$ must be universal for the family of $\mathsf{F}$-local OMQA-languages.

\section{Concluding Remarks}

We have established the nonexistence of universal language for both FO-rewritable and PTIME-tractable OMQA. As a rescue, we also proposed a novel property, called the locality, as an approximation to the FO-rewritability, and proved that there is some language of DEDs which is universal for OMQA with bounded locality. In spite of the unnaturalness of the constructed language, we believe that the proposed property would shed light on finding natural universal languages, as well as on identifying new tractable languages.

%

\bibliographystyle{aaai}
\bibliography{ref}

\end{document}